\documentclass{article}

\usepackage{lmodern}

\usepackage[utf8]{inputenc} 
\usepackage[T1]{fontenc}    
\usepackage{hyperref}       
\usepackage{url}            
\usepackage{booktabs}       
\usepackage{amsfonts}       
\usepackage{nicefrac}       
\usepackage{microtype}      
\usepackage[margin=1.20in]{geometry}

\usepackage{amsmath}
\usepackage{amsthm}
\usepackage{epsfig}
\usepackage{cite}
\usepackage{graphicx}
\graphicspath{{Figures/}}
\usepackage{subfigure}
\usepackage{url}
\usepackage{stfloats}
\usepackage[square, numbers]{natbib}
\usepackage{algorithm}
\usepackage{algpseudocode}
\usepackage{amsopn}    
\usepackage{adjustbox}
\usepackage{hyperref}
\usepackage{wrapfig}

\newtheorem{theorem}{\bf Theorem}
\newtheorem{rem}{\bf Remark}

\newtheorem{lemma}{\bf Lemma}

\newtheorem{prop}{\bf Proposition}

\newcommand{\kld}{\mathbb{D}}

\makeatletter
\def\hlinewd#1{%
  \noalign{\ifnum0=`}\fi\hrule \@height #1 \futurelet
   \reserved@a\@xhline}
\makeatother

\title{\bf CoNES: Convex Natural Evolutionary Strategies}

\author{\textbf{Sushant Veer}\\
Mechanical and Aerospace Engineering \\
Princeton University\\
\texttt{sveer@princeton.edu} \and \textbf{Anirudha Majumdar}\\
Mechanical and Aerospace Engineering \\
Princeton University\\
\texttt{ani.majumdar@princeton.edu}}

\date{}

\input{irom.sty}

\begin{document}

\maketitle

\begin{abstract}
We present a novel algorithm -- \emph{convex natural evolutionary strategies (CoNES)} -- for optimizing high-dimensional blackbox functions by leveraging tools from convex optimization and information geometry. CoNES is formulated as an efficiently-solvable convex program that adapts the evolutionary strategies (ES) gradient estimate to promote rapid convergence. The resulting algorithm is \emph{invariant} to the parameterization of the belief distribution. Our numerical results demonstrate that CoNES vastly outperforms conventional blackbox optimization methods on a suite of functions used for benchmarking blackbox optimizers. Furthermore, CoNES demonstrates the ability to converge faster than conventional blackbox methods on a selection of OpenAI's MuJoCo reinforcement learning tasks for locomotion.
\end{abstract}

\section{Introduction}

Policy optimization in reinforcement learning (RL) can be posed as a blackbox optimization problem: given access to a ``blackbox'' in the form of a simulator or robot hardware, find a setting of policy parameters that maximizes rewards. This perspective has led to significant recent interest from the RL community towards scaling blackbox optimization methods and has catapulted the use of blackbox optimizers from low-dimensional hyperparameter tuning \cite{Golovin17,Hutter19} to training deep neural networks (DNNs) with thousands of parameters \cite{Salimans17,Choromanski19a,Choromanski19b,Liu19,Conti18,Mania18}. Despite these promising advances, the sample complexity of blackbox methods remains high and is the subject of ongoing research. 

In this paper we study a class of blackbox optimization methods called evolutionary strategies (ES) \cite{Rechenberg73,Salimans17}. ES methods maintain a belief distribution on the domain of candidates. At each iteration, a batch of candidates is sampled from this distribution and their fitness is evaluated. These fitness scores are used to obtain a Monte-Carlo (MC) estimate of the loss function's gradient with respect to the parameters of the belief distribution. In the domain of ES for RL, approaches that adapt the sampling rate from the belief distribution and reuse samples from previous iterations have been proposed to improve the sample complexity \cite{Choromanski19a,Choromanski19b}. However, standard ES methods are not invariant to re-parameterizations of the belief distribution. Hence, the choice of belief parameterization (e.g., encoding the covariance as a symmetric positive definite matrix vs. a Cholesky decomposition) can affect the rate of convergence and cause undesirable behavior (e.g., oscillations) \cite{Wierstra2014}. In contrast, ES techniques based on the \emph{natural gradient} \cite{Amari98,Sun09,Wierstra2014} are \emph{parameterization invariant} and can demonstrate improved sample efficiency. However, these methods have not been thoroughly exploited in RL due to the difficulties in computing the natural gradient for high-dimensional problems; in particular, the challenging estimation of the Fisher information matrix is necessary for computing the natural gradient. 

In this paper, we present a novel algorithm -- \emph{convex natural evolutionary strategies (CoNES)} -- that leverages results on the natural gradient \cite{Amari98,Sun09,Wierstra2014} from information geometry \cite{Amari16} and couples them with powerful tools from convex optimization (e.g., second-order cone programming \cite{Boyd04} and geometric programming \cite{Boyd07}) to promote rapid convergence. 
In particular, CoNES refines a crude gradient estimate by transforming it through a convex program that searches for the direction of steepest ascent in a KL-divergence ball around the current belief distribution. 
The relationship to natural evolutionary strategies (NES) \cite{Wierstra2014} comes from the fact that the limiting solution of the KL-constrained optimization problem (as the ``radius'' of the KL-divergence ball shrinks to zero) corresponds to the natural gradient. However, in contrast to NES \cite{Wierstra2014}, CoNES circumvents the estimation of the Fisher information matrix by directly solving the convex KL-constrained optimization problem. 
\begin{wrapfigure}{r}{0.45\textwidth}
    \begin{center}
        \vspace{-5mm}
        \includegraphics[trim=70 55 75 110,clip,width=0.4\textwidth]{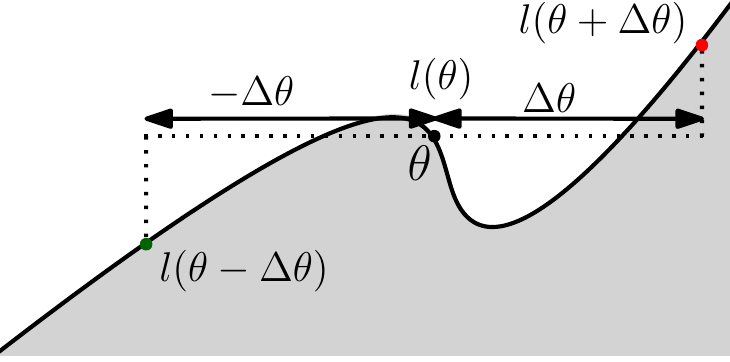}
    \end{center}
    \vspace{-3mm}
    \caption{\small Illustration demonstrating the importance of accounting for the step length for choosing the update direction. At the belief distribution expressed in the coordinates $\theta$, if we follow the negative of the gradient direction (right), then, with the step size $\Delta \theta$, the loss increases. However, accounting for the step size while choosing the direction, we would go left and the loss would decrease. \label{fig:step-length}}
    \vspace{-3.5mm}
\end{wrapfigure}
Furthermore, tuning the radius of the KL-divergence ball facilitates better alignment of the update direction with the update step size, yielding faster convergence than NES (which provides the steepest ascent direction for infinitesimal steps lengths); see Fig.~\ref{fig:step-length} for an illustration that demonstrates the importance of accounting the step length for choosing the update direction.

Our theoretical results establish that CoNES is \emph{invariant} to the parameterization of the belief distribution (e.g., encoding the covariance as a symmetric positive definite matrix or a Cholesky decomposition does not affect the solution of the CoNES optimization problem). Parameterization invariance ensures that we are working with the intrinsic mathematical object (i.e., probability distribution) and the specific encoding of these objects do not affect the outcome. Moreover, CoNES is agnostic to the method that generates the crude gradient estimate and can thus be potentially combined with various existing ES methods, such as \cite{Salimans17,Choromanski19a,Choromanski19b}.  Through our numerical results we demonstrate that CoNES vastly outperforms various conventional blackbox optimizers on a suite of 5000-dimensional benchmark functions for blackbox optimizers: \texttt{Sphere, Rosenbrock, Rastrigin}, and \texttt{Lunacek}. We also demonstrate the improved sample complexity achieved by CoNES on the following OpenAI MuJoCo RL tasks: \texttt{HalfCheetah-v2, Walker2D-v2, Hopper-v2}, and \texttt{Swimmer-v2}.

\section{Related Work}

{\bf Blackbox optimization.} Various engineering problems require optimizing systems for which the governing mechanisms are not explicitly known; e.g., system identification of complex physical systems \cite{Amaran16} and mechanism design \cite{Audet16}. Blackbox optimization techniques such as Nelder-Mead \cite{Nelder65}, evolutionary strategies (ES) \cite{Rechenberg73}, simulated annealing \cite{Kirkpatrick83}, genetic algorithms \cite{Holland92}, the cross-entropy method \cite{De05}, and covariance matrix adaptation (CMA) \cite{Hansen16} were developed to address such problems. Recently, the growing potential of these methods for training control policies with reinforcement learning \cite{Salimans17,Mania18,Choromanski19a,Choromanski19b,Liu19,Conti18,Chatzilygeroudis17,Ha19} has reignited interest in blackbox optimizers. In this paper, we will primarily consider the class of blackbox optimizers that fall under the purview of ES.

\noindent{\bf Evolutionary strategies for reinforcement learning.} In RL tasks, the advantages of ES -- high parallelizability, better robustness, and richer exploration -- were first demonstrated in \cite{Salimans17}. Spurred by these findings, a plethora of recent developments aimed at improving ES for RL have emerged, some of which include: explicit novelty search regularization to avoid local minima \cite{Conti18}, robustification of ES and efficient re-use of prior rollouts \cite{Choromanski19a}, and adaptive sampling for the ES gradient estimate \cite{Choromanski19b}. We remark that all the above papers focus on improving the ES MC gradient estimator. In contrast, this paper presents a method that \emph{refines} the ES gradient estimate -- regardless of where that estimate comes from -- by solving a convex program. 

\noindent{\bf Natural gradient.} Our method is directly motivated by the concept of the natural gradient \cite{Amari16}. The application of natural gradient in learning was initially pioneered in \cite{Amari98} and was later demonstrated to be effective for RL \cite{Kakade02}, deep learning with backpropagation \cite{Pascanu13}, and blackbox optimization with ES \cite{Sun09,Wierstra2014}. However, the latent potential of the natural gradient has not been completely realized due to the difficulty in estimation of the Fisher information matrix. Much of the prior work employing natural gradient has focused on efficient estimation or computation of the Fisher information matrix \cite{Wu17,Sun09,Pascanu13}. In contrast, CoNES does not work directly with the Fisher information matrix. Instead, we approximate the update direction by solving a convex program that maximizes the loss while being constrained to a KL-divergence ball around the current belief distribution; as the radius of the KL-divergence ball goes to zero, the limiting solution of this convex program corresponds to the natural gradient (see Proposition \ref{prop:nat-grad}). 

\noindent{\bf Trust-regions for blackbox optimization.} Recent work on trust region methods for blackbox optimizers \cite{Liu19,Miyashita18,Abdolmaleki17} performs updates on the belief distribution by optimizing the loss on a KL-divergence ball. However, \cite{Abdolmaleki17,Miyashita18} perform the constrained optimization on a \emph{discretization} of the belief distribution. The approach in \cite{Liu19} computes the KL-divergence for each dimension individually and bounds their maximum; the resulting optimization problem is approximated via a clipped surrogate objective similar to proximal policy optimization (PPO) \cite{Schulman17}. In contrast, we \emph{exactly} solve a KL-constrained problem whose solution approximates the natural gradient (as outlined above and formally discussed in Section \ref{subsec:background-nat-grad}) using powerful tools from convex optimization (e.g., second-order cone programming and geometric programming).

\section{Notation}

We denote a blackbox loss function by $\hat{l}:\mathcal{X}\to\mathbb{R}$ with $\mathcal{X}\subseteq\mathbb{R}^m$ as its domain. Let $P$ be a distribution on the domain $\mathcal{X}$ that signifies our \emph{belief} of where the optimal candidate for $\hat{l}$ resides. We assume that $P$ belongs to the statistical manifold $\mathcal{P}$ \cite{Suzuki14} which is a Riemannian manifold \cite{Petersen06} of probability distributions. Any point $P\in\mathcal{P}$ is expressed in the coordinates $\theta\in\mathbb{R}^n$. Rather than optimizing $\hat{l}$ directly, we will work with the loss function $l:\mathcal{P}\to\mathbb{R}$ which provides the expected loss $P \mapsto \mathbb{E}_{x\sim P}[\hat{l}(x)]$ under the belief distribution $P$. When referring to the manifold in a coordinate-free setting, we express the loss as $l:\mathcal{P}\to\mathbb{R}$, whereas, when we work with a particular coordinate system on $\mathcal{P}$, we express the loss as $l:\mathbb{R}^n\to\mathbb{R}$; the abuse of notation creates no confusion as it will always be clear from context.

The (Euclidean) gradient operator is denoted by $\nabla$; the natural gradient operator is denoted by $\tilde{\nabla}$; and the solution of CoNES is denoted by $\hat{\nabla}$. The KL-divergence between two distributions is denoted by $\kld(\cdot || \cdot)$ and the Euclidean inner product between two vectors is denoted by $\langle \cdot, \cdot \rangle$.

\section{Background}

\subsection{Natural Gradient}
\label{subsec:background-nat-grad}

It is a commonly-held belief that the steepest ascent direction for a loss function $l:\mathcal{P}\to\mathbb{R}$ is given by its gradient $\nabla l$. However, this is only true if the domain $\mathcal{P}$ is expressed in an orthonormal coordinate system in a Euclidean space. If the space $\mathcal{P}$ admits a Reimannian manifold \cite{Petersen06} structure, the steepest ascent direction is then given by the \emph{natural gradient} $\tilde{\nabla} l$ instead \cite[Section~12.1.2]{Amari16}. Besides providing the steepest ascent direction on $\mathcal{P}$, the natural gradient possesses various attractive properties: (a) natural gradient is independent of the choice of coordinates $\theta$ on the statistical manifold $\mathcal{P}$; (b) natural gradient avoids saturation due to sigmoidal activation functions \cite[Theorem~12.2]{Amari16}; (c) online natural-gradient learning is asymptotically Fisher efficient, i.e., it asymptotically approaches equality of the Cram\'er-Rao bound \cite{Amari98}. These qualities lay the foundation of our interest in leveraging the natural gradient in learning applications. In the rest of this section we will present two explicit characterizations of the natural gradient relevant to this paper.

Let $F(\theta)$ be the Fisher information matrix for the Reimannian manifold of distributions $\mathcal{P}$ described in the coordinates $\theta$; e.g., Gaussian distributions can be expressed in the coordinates $\theta=(\mu,$\texttt{vec $\circ$ upper-triangle}$(\Sigma))$ where $\mu$, $\Sigma$ denote the mean and the covariance, respectively. The natural gradient then satisfies the following relation with the Euclidean gradient:
\begin{equation}\label{eq:nat-grad-fish}
\tilde{\nabla} l(\theta) = F(\theta)^{\rm -1} \nabla l(\theta) \enspace.
\end{equation}

For the second characterization of the natural gradient we will need the Fisher-Rao norm $\|\cdot\|_F:\mathcal{P}\to[0,\infty)$ defined as $\|\theta\|_F:= \sqrt{\langle \theta, F(\theta)\theta \rangle}$ \cite[Definition~2]{Liang17}. Using this norm we can express the natural gradient as follows:

\begin{prop} \label{prop:nat-grad}
\emph{\textbf{[Adapted from \cite[Proposition~1]{Ollivier17}]}} Let $\mathcal{P}$ be a statistical manifold, each point of which is a probability distribution $P_\theta$ parameterized by $\theta$. Let $l:\mathcal{P}\to\mathbb{R}$ be a loss function which maps a probability distribution $P_\theta$ to a scalar. Then, the natural gradient $\tilde{\nabla} l(\theta)$ of the loss function computed at any $\theta$ satisfies:
\begin{flalign}
\frac{\tilde{\nabla}l(\theta)}{\|\tilde{\nabla}l(\theta)\|_F} = \lim_{\epsilon\to 0} ~ \underset{v\in\mathbb{R}^n}{\arg\max}~~& l(\theta+\epsilon v) \label{eq:nat-grad}\\
\text{s.t.}~~~& \kld(P_{\theta+\epsilon v}||P_\theta)\leq \epsilon^2/2 \nonumber \enspace.
\end{flalign}
\end{prop}

Proposition~\ref{prop:nat-grad} states that the natural gradient is aligned with the direction $v$ which maximizes the loss function in an infinitesimal KL-divergence ball around the current distribution $P_\theta$. To avoid confusion, it is worth clarfiying that the maximization in Proposition~\ref{prop:nat-grad} computes the natural gradient which can then be passed to a gradient-based optimizer to \emph{minimze the loss}.

\begin{rem}\label{rem:linear-nat-grad}
Proposition~\ref{prop:nat-grad} also holds true for the linear approximation of the loss function $l(\theta+\epsilon v)$ at $\theta$. Intuitively, the reason for this is that the linear approximation locally converges to the loss function for arbitrarily small $\epsilon>0$.
\end{rem}

\subsection{Natural Evolutionary Strategies}
\label{subsec:background-ES}

The evolutionary strategies (ES) framework performs a Monte-Carlo estimate of the gradient of the loss with respect to the belief distribution \cite[Section~2]{Wierstra2014}:
\begin{equation}\label{eq:grad-ES}
\nabla l(\theta) = \nabla \underset{x\sim P_\theta}{\mathbb{E}}[\hat{l}(x)] = \underset{x\sim P_\theta}{\mathbb{E}} [\hat{l}(x) \nabla \ln P_\theta(x) ] \enspace.
\end{equation} 
This gradient estimate is then supplied to a gradient-based optimizer to update the belief distribution. Note that \eqref{eq:grad-ES} provides an estimate of the Euclidean gradient. Instead of using the Euclidean gradient \eqref{eq:grad-ES}, Natural Evolutionary Strategies (NES) \cite{Wierstra2014,Sun09} estimates the natural gradient by transforming the Euclidean gradient estimate \eqref{eq:grad-ES} through \eqref{eq:nat-grad-fish}.

\section{Convex Natural Evolutionary Strategies}

Despite the various advantages offered by the natural gradient, the computationally expensive estimation of the Fisher information matrix $F(\theta)$ and its inverse makes it difficult to scale to very high-dimensional problems. Proposition~\ref{prop:nat-grad} offers an alternative to compute the natural gradient while obviating the need to estimate $F(\theta)$; however, \eqref{eq:nat-grad} is a challenging non-convex optimization problem. To develop CoNES we ``massage'' \eqref{eq:nat-grad} into an efficiently-solvable convex program.

We begin by relaxing relaxing the requirement $\lim \epsilon \to 0 $ and instead choosing a fixed $\epsilon>0$, resulting in the following optimization problem:\footnote{Without loss of generality, we are replacing $\epsilon^2/2$ with $\epsilon^2$.}
\begin{flalign}\label{eq:OPT-1}
v^*(\theta) \in \arg\max_v\{l(\theta+\epsilon v)~|~\kld(P_{\theta+\epsilon v}||P_\theta)\leq \epsilon^2, v\in\mathbb{R}^n\}, 
\end{flalign}
where $\epsilon$ is now a hyperparameter which can be as large as necessary. Using $v^*(\theta)$ as the update direction could yield faster convergence than $\tilde{\nabla}l(\theta)$. This may seem counter-intuitive because the natural gradient is the steepest ascent direction, as discussed in Section~\ref{subsec:background-nat-grad}; however, it is worth noting that this holds true only for an infinitesimal step length. The flexibility of choosing an $\epsilon$ permits us to align the search for the steepest ascent direction with the desired step-length of the update, yielding rapid convergence; see Fig.~\ref{fig:step-length} for an illustration.

We are interested in settings where the landscape of the loss function $l$ is unknown and querying loss values of individual candidates is expensive. Even if the analytical form of $l$ was available to us, \eqref{eq:OPT-1} may be a non-convex problem and hence challenging to solve. To make this problem more tractable, we perform a Taylor expansion of the loss function $l(\theta+\epsilon v) \approx l(\theta) + \langle \nabla l(\theta), \epsilon v \rangle$ and work with the following optimization problem:
\begin{flalign}\label{eq:OPT-2}
v^*(\theta) \in \arg\max_v\{l(\theta)+\langle \nabla l(\theta), \epsilon v \rangle~|~\kld(P_{\theta+\epsilon v}||P_\theta)\leq \epsilon^2, v\in\mathbb{R}^n\}.
\end{flalign}
In \eqref{eq:OPT-2}, $l(\theta)$ is a constant offset which does not affect the choice of $v$ and can hence be ignored. Further, we denote $\delta\theta:=\epsilon v$ and restate \eqref{eq:OPT-2} as:
\begin{flalign}\label{eq:OPT}
\hat{\nabla}l(\theta;\epsilon) \in \arg\max_{\delta\theta}\{\langle \nabla l(\theta), \delta\theta \rangle~|~\kld(P_{\theta+\delta\theta}||P_\theta)\leq \epsilon^2, \delta\theta\in\mathbb{R}^n\}.
\end{flalign}

Despite these relaxations, the optimization problem \eqref{eq:OPT} may still be intractable due to the lack of convexity of the feasible set. However, in the following theorem we establish for the Gaussian family of probability distributions that \eqref{eq:OPT} is convex and can be solved in \emph{polynomial time}.
\begin{theorem}
\label{thm:cvx-dist}
The optimization \eqref{eq:OPT} is:
\begin{itemize}
\item a semidefinite program (SDP) with an additional exponential cone constraint if $\mathcal{P}$ is the space of Gaussian distributions;
\item a second-order cone program (SOCP) with an additional exponential cone constraint if $\mathcal{P}$ is the space of Gaussian distributions with diagonal covariance.
\end{itemize}
\end{theorem}
\begin{proof}
As the objective function of \eqref{eq:OPT} is linear, we only need to verify the convexity of the feasible set. 
We will first consider the case when $\mathcal{P}$ is the space of Gaussian distributions. Let $P_{\theta+\delta\theta}=\mathcal{N}(\mu,\Sigma)$ and $P_\theta = \mathcal{N}(\mu_0,\Sigma_0)$. Then:
\begin{equation}\label{eq:cont-gauss-kld}
\kld(P_{\theta+\delta\theta}||P_\theta) = \frac{1}{2} \bigg(\text{Tr}(\Sigma_0^{-1}\Sigma) + (\mu-\mu_0)^{\rm T} \Sigma_0^{-1} (\mu-\mu_0) - \log\det(\Sigma) + \log \det(\Sigma_0) - n \bigg)
\end{equation}
which is convex because $\text{Tr}(\Sigma_0^{-1}\Sigma)$ is linear, $(\mu-\mu_0)^{\rm T} \Sigma_0^{-1} (\mu-\mu_0)$ is positive-definite quadratic, and $-\log\det(\Sigma)$ is convex. Finally, noting that $\log\det$ constraints can be formulated as an SDP with an additional exponential cone constraint \cite{logdet} completes the proof of this part.

Now we consider the family of Gaussian distributions $P_{\theta+\delta\theta}=\mathcal{N}(\mu,\Sigma)$ and $P_\theta = \mathcal{N}(\mu_0,\Sigma_0)$ with diagonal covariance. We denote the mean as $\mu=(\mu_1,\cdots,\mu_n)$ and $\mu_0=(\mu_{0,1},\cdots,\mu_{0,n})$. The diagonal elements of the covariance $\Sigma$ and $\Sigma_0$ are expressed as $(\sigma_1,\cdots,\sigma_n)$ and $(\sigma_{0,1},\cdots,\sigma_{0,n})$, respectively. Then, the KL-divergence between two distributions in this family is:
\begin{equation}\label{eq:diag-gauss-kld}
\kld(P_{\theta+\delta\theta}||P_\theta) = -\frac{1}{2} \sum_{i=1}^n \bigg(1 + \log \sigma_i^2- \log \sigma_{0,i}^2 - \frac{(\mu_i-\mu_{0,i})^2}{\sigma_{0,i}^2} - \frac{\sigma_i^2}{\sigma_{0,i}^2} \bigg) \enspace.
\end{equation}
From \eqref{eq:diag-gauss-kld}, it follows that the problem \eqref{eq:OPT} for this family of distributions is an SOCP with an additional exponential cone constraint (that arises from the $\log$ terms), completing the proof.
\end{proof}

\begin{wrapfigure}{r}{0.4\textwidth}
    \begin{center}
        \vspace{-5mm}
        \includegraphics[width=0.4\textwidth]{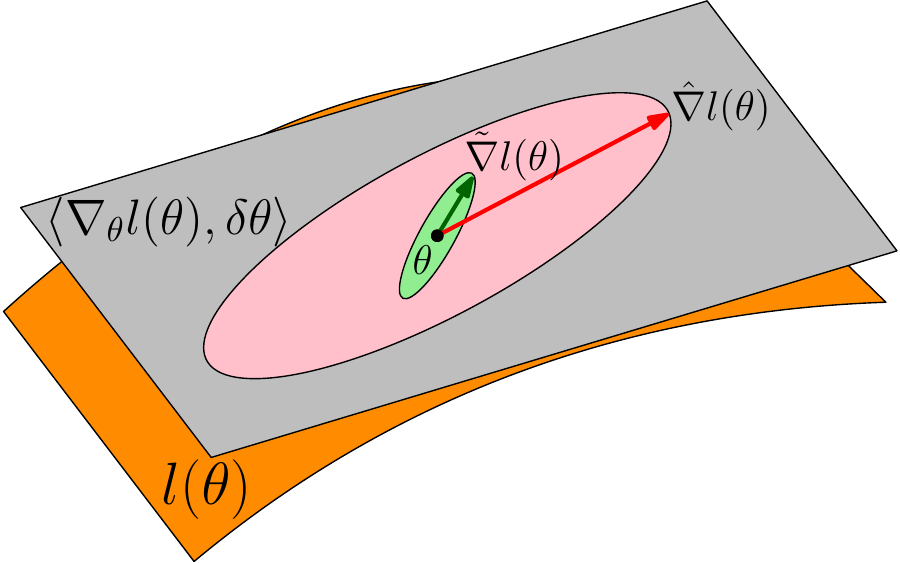}
    \end{center}
    \vspace{-3mm}
    \caption{\small Geometric illustration of CoNES. \label{fig:CoNES}}
    \vspace{-5mm}
\end{wrapfigure}
Restricting the class of belief distributions to those in Theorem~\ref{thm:cvx-dist} gives rise to CoNES: a family of convex programs that draws motivation from the concept of the natural gradient to transform the Euclidean gradient. To geometrically visualize CoNES, consider the illustration in Fig.~\ref{fig:CoNES}. The orange surface is the loss landscape and the gray surface is the linearization of the loss at the point denoted by $\theta$; in differential geometric terms, the orange surface is more accurately characterized as the manifold given by the graph of the loss $l(\theta)$ while the gray surface is the manifold's tangent space at $(\theta, l(\theta))$. The green arrow represents the solution of CoNES for a KL-divergence ball (light green region) with a very small $\epsilon$ which can also be regarded as the natural gradient (modulo the norm) at $\theta$ by Remark~\ref{rem:linear-nat-grad}. The red arrow is the solution of CoNES for a KL-divergence ball (light red region) with a larger $\epsilon$. Note that this figure is an illustration; the KL-divergence balls may not necessarily manifest in the depicted shapes. The NES gradient is the sharpest ascent direction for an infinitesimal step size, but, it may not be ideal for a larger step size. With CoNES, we can tune the scalar parameter $\epsilon$ to better align the update direction with the gradient-based optimizer's step size (learning rate), yielding faster updates. Indeed, the choice of $\epsilon$ is important to the performance of CoNES as demonstrated in our numerical results in Section~\ref{subsec:RL-results}. The mechanism for selecting (or adapting) the hyperparameter $\epsilon$ is beyond the scope of this paper and will be explored in our future work.

The psuedo-code for our implementation of CoNES as a blackbox optimizer is detailed in Algorithm~\ref{alg:CoNES}. We use the ES gradient estimate (presented in Section~\ref{subsec:background-ES}) as the \textsc{Gradient-Estimator} in Line 5 of Algorithm~\ref{alg:CoNES}; any estimator of the Euclidean gradient, such as \cite{Choromanski19a,Choromanski19b}, can be used here. We use Adam \cite{Kingma14} as our gradient-based optimizer in Line 7; any gradient-based optimizer can be used.

\begin{algorithm}[h]
\caption{CoNES \label{alg:CoNES}}
\small
\begin{algorithmic}[1]
\State Hyperparameters: radius $\epsilon$ of KL-divergence ball, number of candidates $N$ drawn at each iteration
\State Initialize: $\theta\gets\theta_0$, \textsc{Optimizer}
\Repeat
	\State $\{\hat{x}_i\}_{i=1}^N\gets$ Draw $N$ samples from the belief distribution $P_\theta$
	\State $\nabla_\theta l(\theta)\gets$ \textsc{Gradient-Estimator}($\{x_i\}_{i=1}^N,\{\hat{l}(x_i)\}_{i=1}^N$)
	\State $\hat{\nabla}_\theta l(\theta)\gets$ \textsc{CoNES}($\nabla_\theta l(\theta),\epsilon$) \Comment{solve \eqref{eq:OPT}}
	\State $\theta \gets$ \textsc{Optimizer}($\theta$, $\hat{\nabla}_\theta l(\theta)$)
\Until Termination conditions satisfied\\
\Return $\theta$
\end{algorithmic}
\normalsize
\end{algorithm}

\section{Parameterization invariance of CoNES}
An important property of the natural gradient is its independence to the parameterization of the belief distribution; e.g., for Gaussian distributions it does not matter whether we use the covariance matrix or its Cholesky decomposition. The natural gradient inherits this property by construction as the covariant gradient on the statistical manifold \cite{Amari16}. Parameterization invariance ensures that we are working with the intrinsic mathematical objects (probability distributions here) and the specific encoding of these objects will not affect the outcome. From a practical perspective, we derive the benefit of fewer properties to ``engineer''.

A natural question to ask is whether CoNES (Problem \eqref{eq:OPT}) exhibits the same property. Proposition~\ref{prop:nat-grad} ensures that the CoNES optimization exhibits this property in the limit of $\epsilon$ tending to zero, as the update direction then coincides with the natural gradient. However, establishing this property for arbitrary $\epsilon>0$ is not immediately obvious. The rest of this section is dedicated to formally demonstrating that CoNES does indeed exhibit this property.

We will work with the loss function $l$ rather than its linearization with the understanding that if the parameterization invariance holds for an arbitrary function $l$, it will automatically hold for the linear function in \eqref{eq:OPT}. With a slight abuse of notation, we will express the loss function $l:\mathbb{R}^n\to\mathbb{R}$ in the coordinates $\theta$ on the statistical manifold instead of the coordinate-free notation of $l:\mathcal{P}\to\mathbb{R}$. Now we are ready to present the main result of this section:

\begin{theorem}\label{thm:param-invar-coNES}
Consider the optimization problem:
\begin{flalign}
OPT_\theta: l_\theta^* = \max\{l(\theta+\epsilon v_\theta)~|~\kld(P_{\theta+\epsilon v_\theta}||P_\theta)\leq \epsilon^2, v_\theta\in\mathbb{R}^n\}.
\end{flalign}
Let $\Phi:\mathbb{R}^n \to\mathbb{R}^n$ be a smooth invertible mapping which performs a coordinate change from $\theta \mapsto \phi:=\Phi(\theta)$. Consider the following optimization problem $OPT_\phi$ in the new coordinates:\footnote{From a geometric perspective, $\theta$ and $\phi$ are coordinates on the statistical manifold $\mathcal{P}$, either of which can be used to express a distribution $P\in\mathcal{P}$. The directions $v_\theta$ and $v_\phi$ lie in the tangent space $T_P\mathcal{P}$ of $\mathcal{P}$ at $P$.}
\begin{flalign}
OPT_\phi: l_\phi^* = \max\{l\circ\Phi^{-1}(\phi+\epsilon v_\phi)~|~\kld(P_{\phi+\epsilon v_\phi}||P_\phi)\leq \epsilon^2, v_\phi\in\mathbb{R}^n\}.
\end{flalign} 
Then, there exists an invertible mapping $\Phi_v:\mathbb{R}^n\to\mathbb{R}^n$ such that $v_\theta^*\in\arg\max_v OPT_\theta \iff \Phi_v(v_\theta^*)\in\arg\max_v OPT_\phi$, ensuring that $l_\theta^* = l_\phi^*$.
\end{theorem}

Theorem~\ref{thm:param-invar-coNES} shows that expressing the belief distribution $P\in\mathcal{P}$ in different coordinates $\theta$ or $\phi$ provides the same optimal loss and the same set of possible outcomes (upto a bijective mapping). Of course, we cannot ensure that the outcome, i.e., the $\arg\max$ of the CoNES optimization is the same due to the potential lack of uniqueness of the optima; e.g., consider the maximization of $x_1^2+x_2^2$ in $x_1^2+x_2^2 \leq 1$ initialized at $(x_1,x_2)=(0,0)$ -- all directions $v$ from the initial point are equally good. 

Intuitively, Theorem~\ref{thm:param-invar-coNES} holds because the KL-divergence is independent of the parameterization of the distribution \cite[Corollary~4.1]{kullback51}, i.e., for $\theta$, $\phi$, and $\Phi$ as defined in Theorem~\ref{thm:param-invar-coNES}, we have: 
\begin{equation}\label{eq:kld-ind}
    \kld(P_{\theta + \epsilon v_\theta}||P_\theta) = \kld(P_{\Phi(\theta + \epsilon v_\theta)}||P_{\Phi(\theta)}) \enspace.
\end{equation}
To formally prove Theorem~\ref{thm:param-invar-coNES}, we will first establish two lemmas. The first lemma shows the existence of a bijective mapping between $v_\theta$ and $v_\phi$.  
\begin{lemma}\label{lem:phi-v}
Let $\theta$, $\phi$, and $\Phi$ be as defined in Theorem~\ref{thm:param-invar-coNES}. Then, there exists a bijective mapping $\Phi_v:\mathbb{R}^n\to\mathbb{R}^n$, defined as 
\begin{align}\label{eq:phi-v-map}
v_\theta \mapsto \frac{\Phi(\theta + \epsilon v_\theta) - \Phi(\theta)}{\epsilon}.
\end{align}
\end{lemma}
\begin{proof}
First we will check the injectivity of $\Phi_v$:
\begin{align}
\Phi_v(v_\theta) = \Phi_v(v_\theta') \iff \Phi(\theta + \epsilon v_\theta) = \Phi(\theta + \epsilon v_\theta') \iff v_\theta = v_\theta' ~~\text{(since}~\Phi~\text{is~injective)}.
\end{align}

Next, to check the surjectivity of $\Phi_v$, let $v_\phi\in\mathbb{R}^n$ be arbitrary. Then there exists $v_\theta:= (\Phi^{-1}(\Phi(\theta) + \epsilon v_\phi) - \theta)/\epsilon$ which satisfies $\Phi_v(v_\theta) = v_\phi$. 
\end{proof}

In the following remark, we express the result of Lemma~\ref{lem:phi-v} in a form that is more conducive to our forthcoming proof.
\begin{rem}\label{rem:phi-v}
Lemma~\ref{lem:phi-v} ensures that the following relation holds for any $v_\theta\in\mathbb{R}^n$: 
\begin{align}\nonumber
v_\phi=\Phi_v(v_\theta) \iff \phi + \epsilon v_\phi = \Phi(\theta + \epsilon v_\theta) \iff \theta + \epsilon v_\theta = \Phi^{-1}(\phi + \epsilon v_\phi)
\end{align}
where the first equivalence relation holds by using the expression of $\Phi_v$ \eqref{eq:phi-v-map} and the second equivalence relations hold from the bijectivity of $\Phi$.
\end{rem}

\begin{lemma}\label{lem:E-theta-image}
Let $B_\theta:=\{v\in\mathbb{R}^n~|~\kld(P_{\theta + \epsilon v}||P_\theta)\leq \epsilon^2\}$ and $B_\phi:=\{v\in\mathbb{R}^n~|~\kld(P_{\phi + \epsilon v}||P_\phi)\leq \epsilon^2\}$ be the feasible sets of $OPT_\theta$ and $OPT_\phi$, respectively. Let $\Phi_v$ be defined as in Lemma~\ref{lem:phi-v}. Then, $B_\phi = \{\Phi_v(v)~|~v\in B_\theta\}$.
\end{lemma}
\begin{proof}
Let $v_\phi\in\{\Phi_v(v)~|~v\in B_\theta\}$, then there exists a $v_\theta\in B_\theta$ such that $v_\phi = \Phi_v(v_\theta)$. Therefore, Remark~\ref{rem:phi-v} ensures that $\phi + \epsilon v_\phi = \Phi(\theta+\epsilon v_\theta)$, which further gives us:
\begin{align}\label{eq:E-phi-1}
\kld(P_{\phi + \epsilon v_\phi}||P_\phi) = \kld(P_{\Phi(\theta + \epsilon v_\theta)}||P_{\Phi(\theta)}) = \kld(P_{\theta + \epsilon v_\theta}||P_\theta) \leq \epsilon^2,
\end{align}
where the last equality follows from \eqref{eq:kld-ind} and the inequality follows from the fact that $v_\theta\in B_\theta$. From \eqref{eq:E-phi-1} we have that $v_\phi\in B_\phi$ implying that $\{\Phi_v(v)~|~v\in B_\theta\}\subseteq B_\phi$.

Now, let $v_\phi\in B_\phi$. By the surjectivity of $\Phi_v$ from Lemma~\ref{lem:phi-v}, there exists a $v_\theta\in\mathbb{R}^n$ such that $v_\phi = \Phi_v(v_\theta)$. With this, Remark~\ref{rem:phi-v} ensures that $\phi + \epsilon v_\phi = \Phi(\theta+\epsilon v_\theta)$. Hence, using \eqref{eq:kld-ind}, followed by $\phi + \epsilon v_\phi = \Phi(\theta+\epsilon v_\theta)$ gives:  
\begin{align}\label{eq:E-phi-2}
\kld(P_{\theta + \epsilon v_\theta}||P_\theta) = \kld(P_{\Phi(\theta + \epsilon v_\theta)}||P_{\Phi(\theta)}) = \kld(P_{\phi + \epsilon v_\phi}||P_\phi) \leq \epsilon^2
\end{align}
where the last inequality follows from the fact that $v_\phi\in B_\phi$. Therefore, by \eqref{eq:E-phi-2}, we have that $v_\theta \in B_\theta$, which, on combining with the earlier assertion that $v_\phi=\Phi_v(v_\theta)$ implies that $v_\phi\in \{\Phi_v(v)~|~v\in B_\theta\}$. Thereby, ensuring that $B_\phi\subseteq \{\Phi_v(v)~|~v\in B_\theta\}$ and completing the proof.
\end{proof}

\begin{proof}[\bf{Proof of Theorem~\ref{thm:param-invar-coNES}}]
The proof follows from the following chain of arguments:
\begin{align}
v_\theta^*\in\arg\max_v OPT_\theta & \iff l(\theta + \epsilon v_\theta^*)\geq l(\theta + \epsilon v_\theta),~\forall v_\theta\in B_\theta \\
& \iff l\circ \Phi^{-1}(\phi + \epsilon \Phi_v(v_\theta^*)) \geq l\circ \Phi^{-1}(\phi + \epsilon \Phi_v(v_\theta)), ~\forall v_\theta\in B_\theta \label{eq:param-invar-1} \\
& \iff l\circ \Phi^{-1}(\phi + \epsilon \Phi_v(v_\theta^*)) \geq l\circ \Phi^{-1}(\phi + \epsilon v_\phi), ~\forall v_\phi\in B_\phi \label{eq:param-invar-2}\\
& \iff \Phi_v(v_\theta^*)\in\arg\max_v OPT_\phi \enspace,
\end{align}
where \eqref{eq:param-invar-1} follows from Remark~\ref{rem:phi-v} (Lemma~\ref{lem:phi-v}) and \eqref{eq:param-invar-2} follows from Lemma~\ref{lem:E-theta-image}.
Further, because $l(\theta + \epsilon v_\theta^*) = l\circ \Phi^{-1}(\phi + \epsilon \Phi_v(v_\theta^*))$ from Remark~\ref{rem:phi-v}, we get $l_\theta^* = l_\phi^*$.
\end{proof}

\section{Results}
\label{sec:results}

In this section, we use CoNES on two classes of problems: (a) a standard suite of high-dimensional loss functions used to benchmark blackbox optimizers, and (b) a selection of OpenAI Gym's \cite{openaigym} MuJoCo \cite{mujoco} suite of RL tasks. We compare CoNES against existing methods including ES, natural evolutionary strategies (NES), and covariance matrix adaptation (CMA). We custom implemented ES, NES, and CoNES, while CMA is adapted directly from the open-source PyCMA package \cite{Hansen19}; our code is accessible at: \href{https://github.com/irom-lab/conES}{https://github.com/irom-lab/CoNES}. 

The family of Gaussian belief distributions with diagonal covariance is used for ES, NES, and CoNES. This family of belief distributions permits the implementation of NES \emph{exactly} (i.e., without having to numerically estimate the Fisher information matrix \cite{Sun09}) for high-dimensional problems, serving as a strong baseline to compare CoNES against. For CMA, PyCMA's default family of belief distributions -- Gaussian distributions with non-diagonal covariance -- is used. For ES, NES, and CoNES we compute an estimate of the gradient direction and pass it to the Adam optimizer \cite{Kingma14} to update the belief distribution. For each of these methods we perform antithetic sampling and rank-based fitness transformation \cite{Salimans17}. Unlike \cite{Salimans17}, we also update the variance of the belief distribution; we circumvent the non-negativeness constraint of the variance by updating the $\log$ of variance with the Adam optimizer instead. The resulting convex optimization problems for CoNES are solved using the CVXPY package \cite{CVXPY} and the MOSEK solver \cite{MOSEK}.

\subsection{Benchmark Functions}
\label{subsec:benchmark-results}

We first test our approach on four $5000$-dimensional functions: \texttt{Sphere, Rosenbrock, Rastrigin,} and \texttt{Lunacek} \cite{Hansen09} which are provided in Appendix~\ref{app:benchmark}. These functions are commonly-used benchmarks for blackbox optimization methods \cite{COCO,NEVERGRAD}. Hyperparameters for ES, NES, and CoNES are shared across all problems (see Appendix~\ref{app:hyperparams}) while the hyperparameters of CMA are the default values chosen by PyCMA. Training for these benchmark functions was performed on a desktop with a 3.30 GHz Intel i9-7900X CPU with 10 cores and 32 GB RAM. Fig.~\ref{fig:benchmark-loss} plots the average and standard deviation (shaded region) of the loss curves across 10 seeds. The rapid drop of the loss for CoNES demonstrates significant benefits in terms of the sample complexity over other methods. Fig.~\ref{fig:benchmark-stepsize} shows that the step size for CoNES is smaller than ES and NES, which coupled with its lower loss implies that the update direction for CoNES is more accurate than ES and NES. The run-time for a single seed is $\sim$1 minute for ES and NES, $\sim$5 minutes for CoNES, and $\sim$35 minutes for CMA.

\begin{figure*}[t]
\vskip -10pt
\centering
\includegraphics[width=0.24\textwidth]{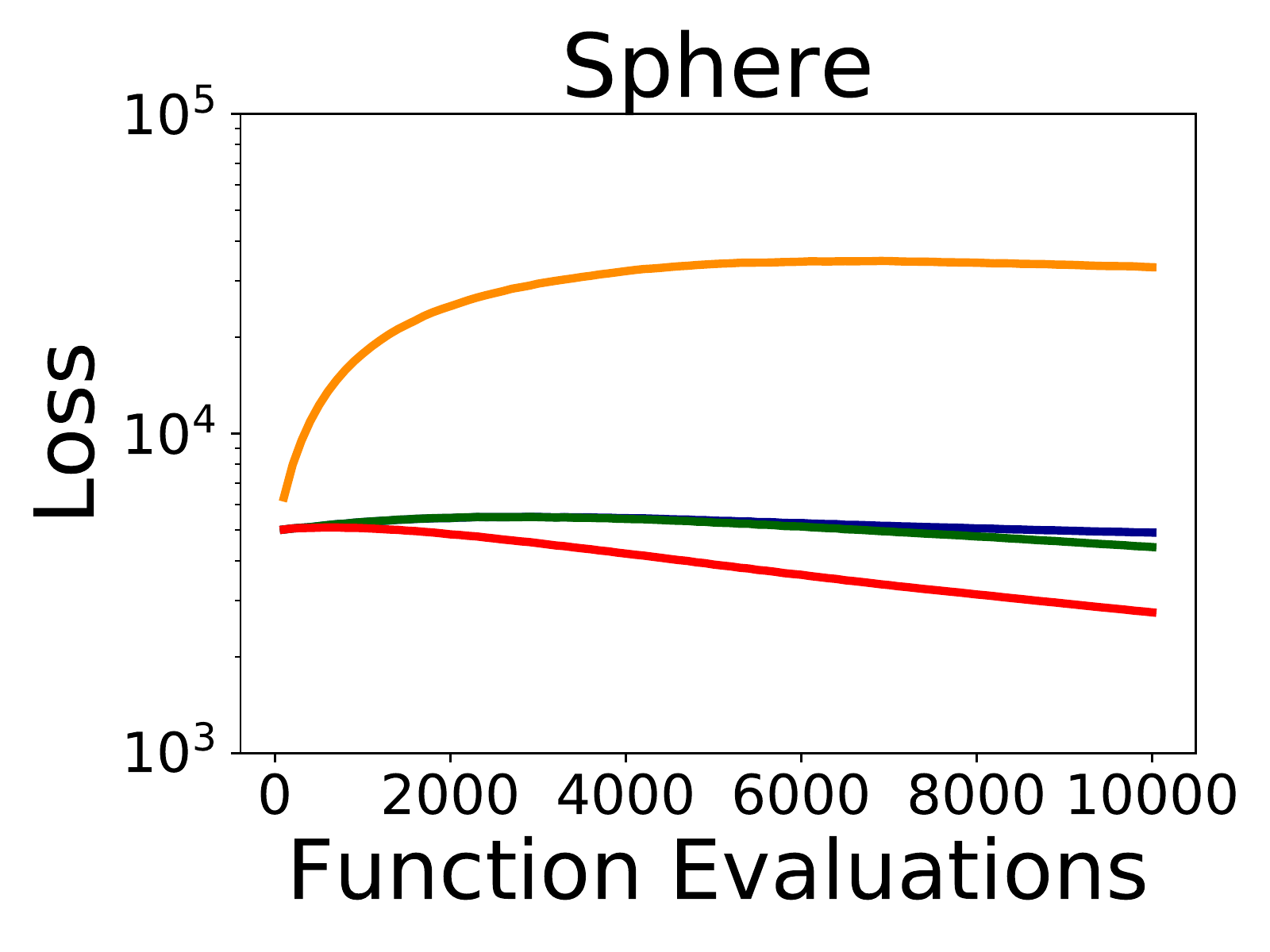}
\includegraphics[width=0.24\textwidth]{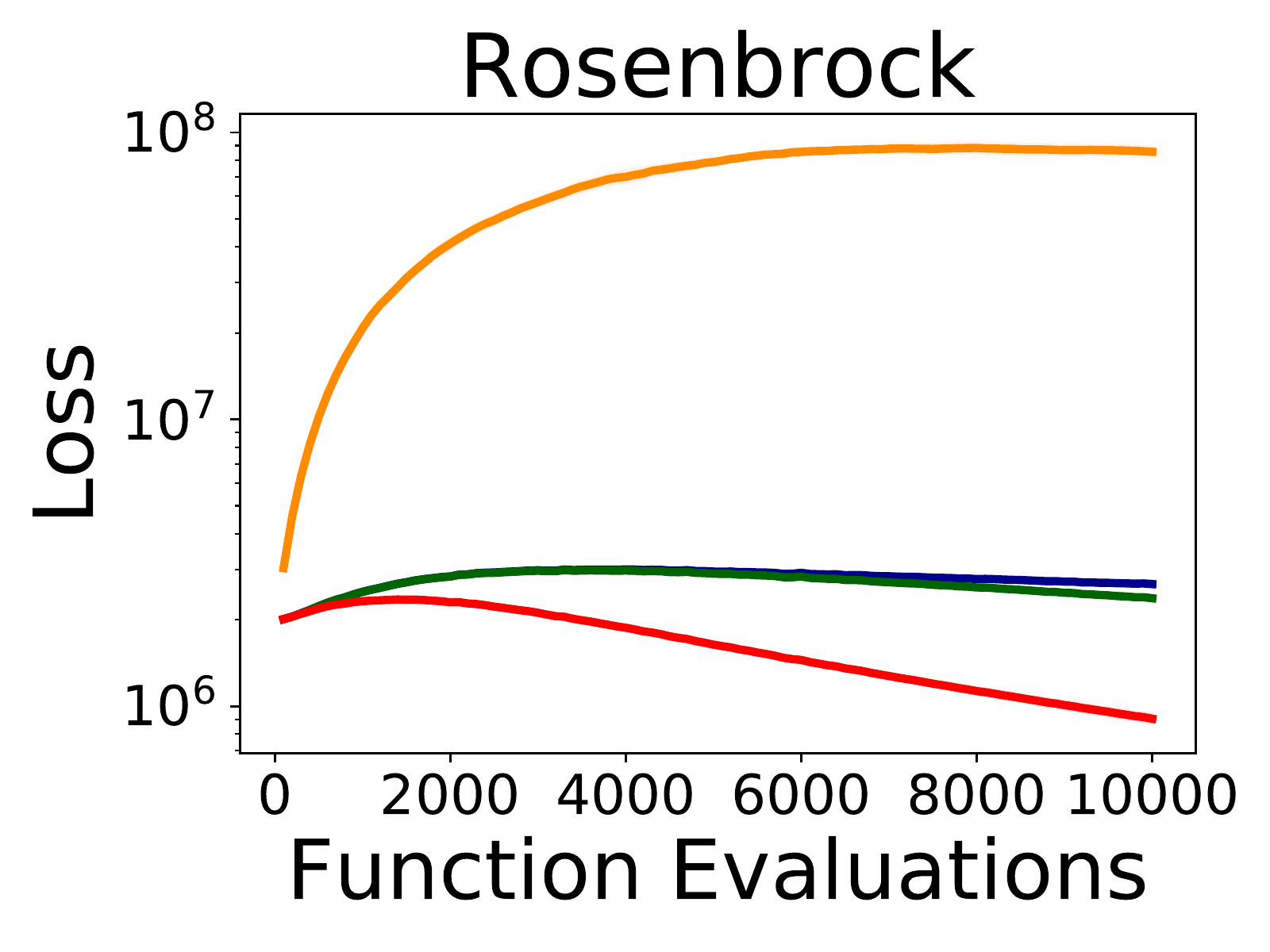}
\includegraphics[width=0.24\textwidth]{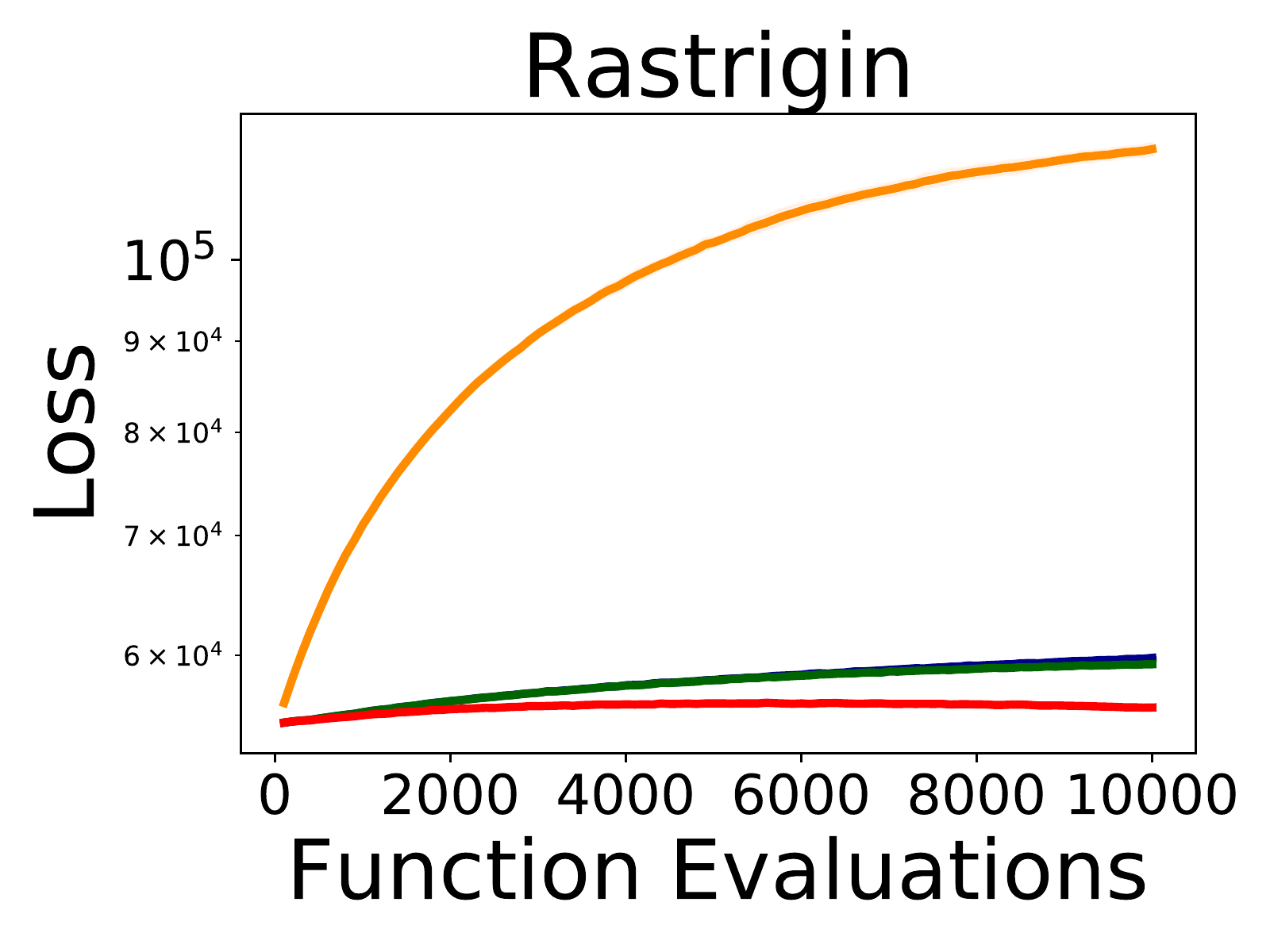}
\includegraphics[width=0.24\textwidth]{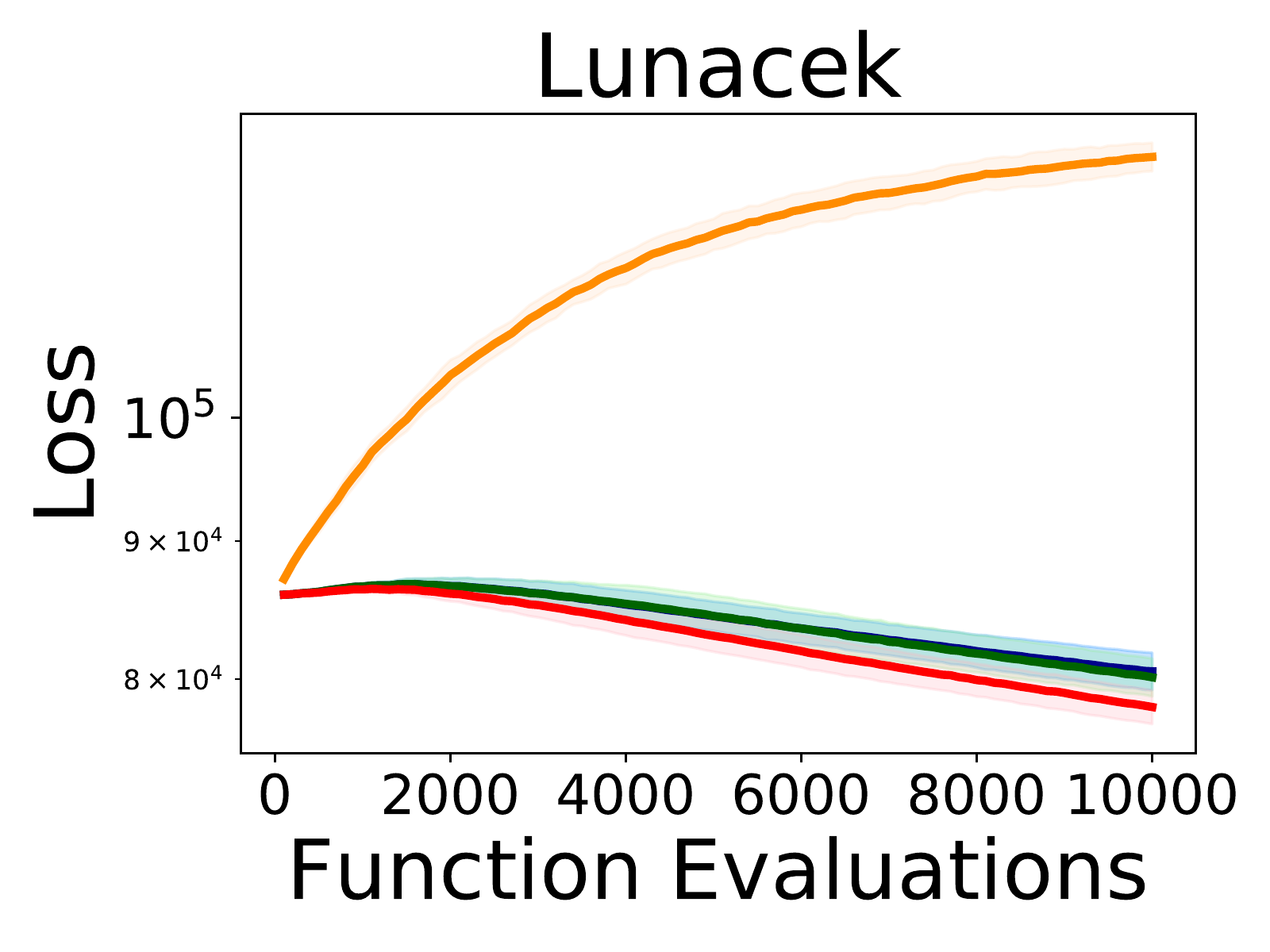}
\includegraphics[width=0.5\textwidth]{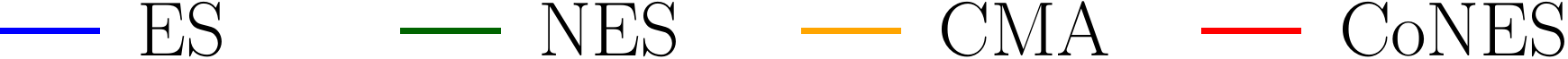}
\vskip -5pt
\caption{\small Average loss (solid curve) with standard deviation (shaded region) across 10 seeds for ES, NES, CMA, and CoNES on \texttt{Sphere, Rosenbrock, Rastrigin,} and \texttt{Lunacek}. \label{fig:benchmark-loss}}
\end{figure*}

\begin{figure*}[t]
\vskip -10pt
\centering
\includegraphics[width=0.24\textwidth]{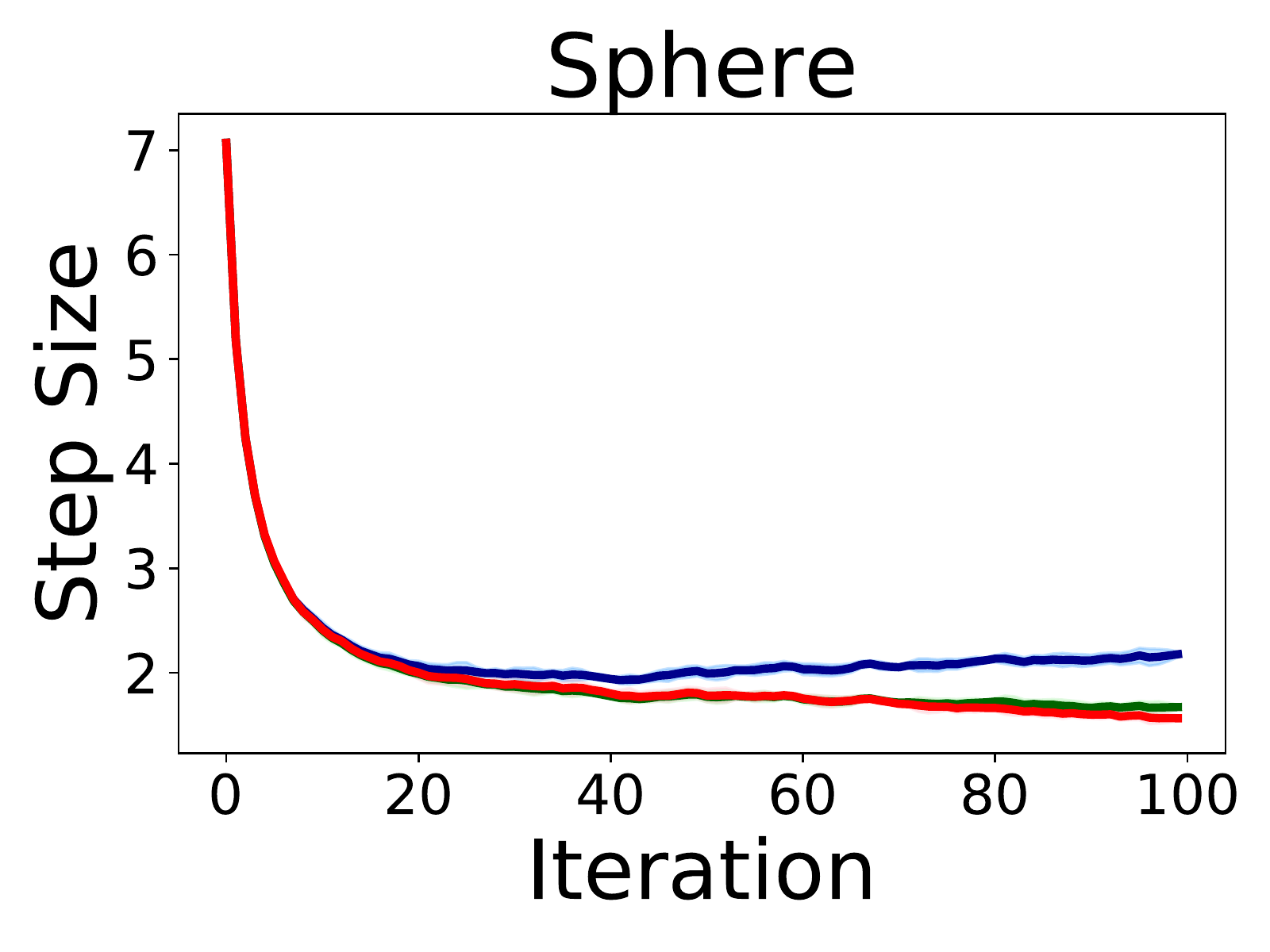}
\includegraphics[width=0.24\textwidth]{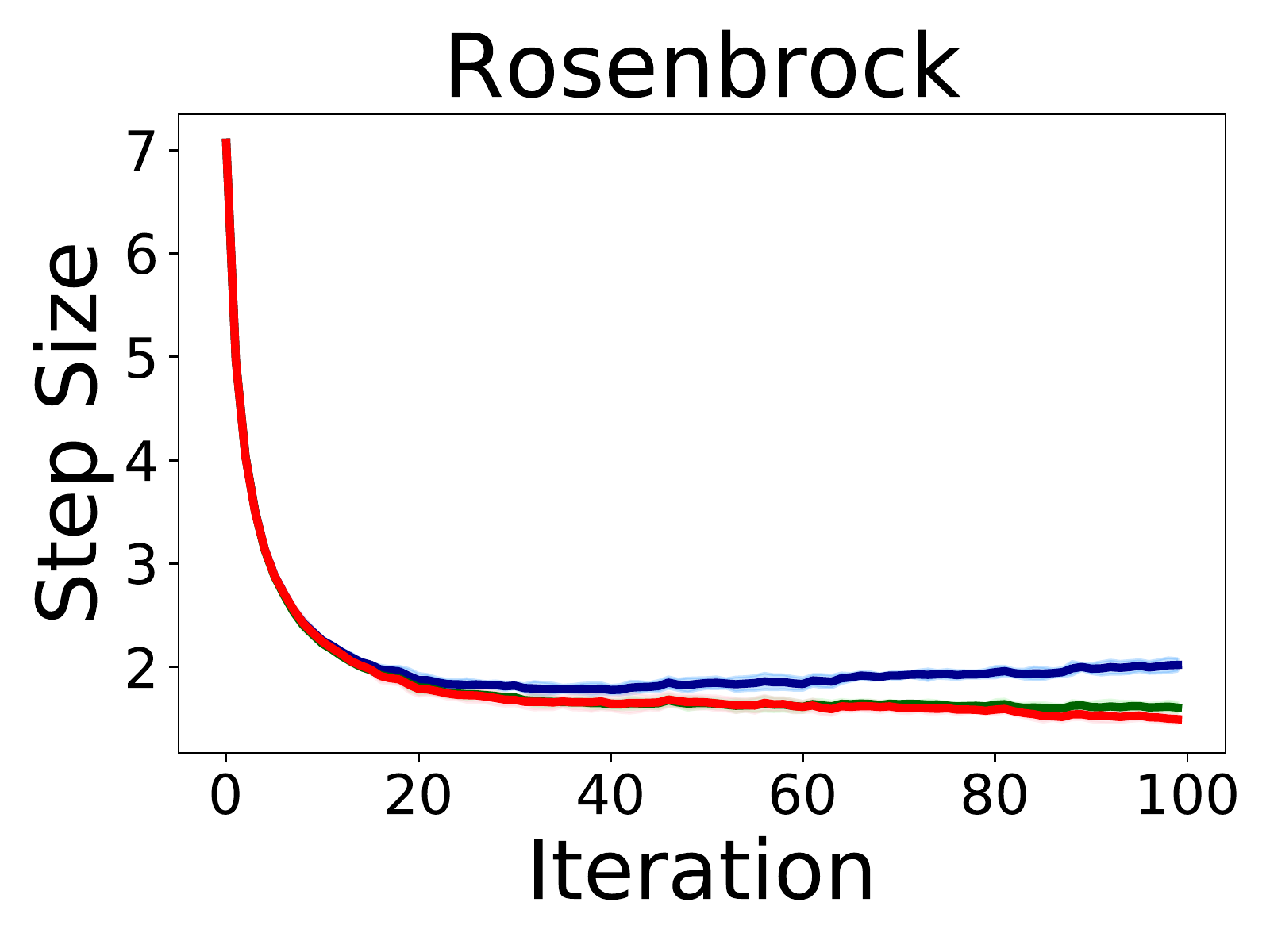}
\includegraphics[width=0.24\textwidth]{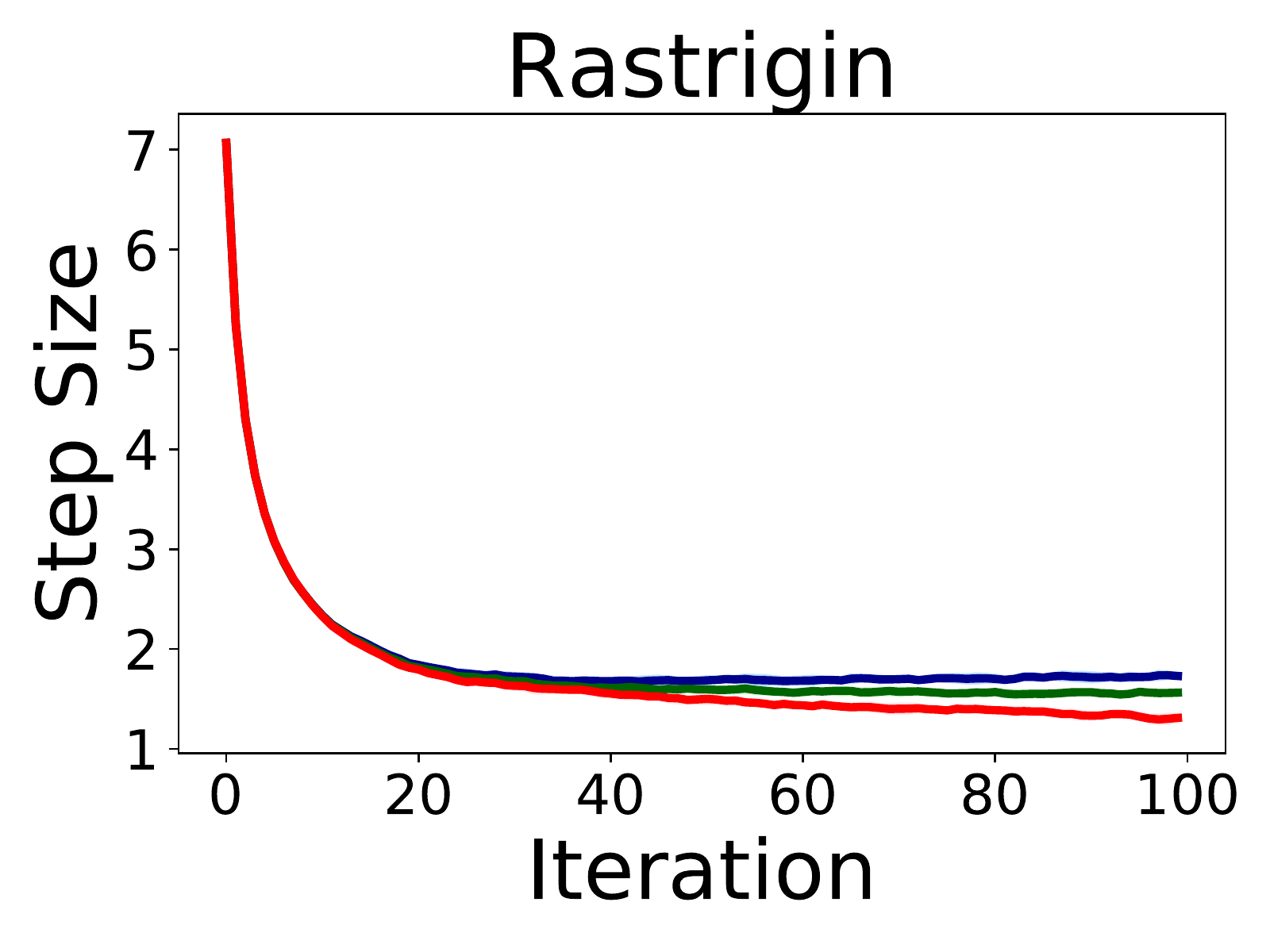}
\includegraphics[width=0.24\textwidth]{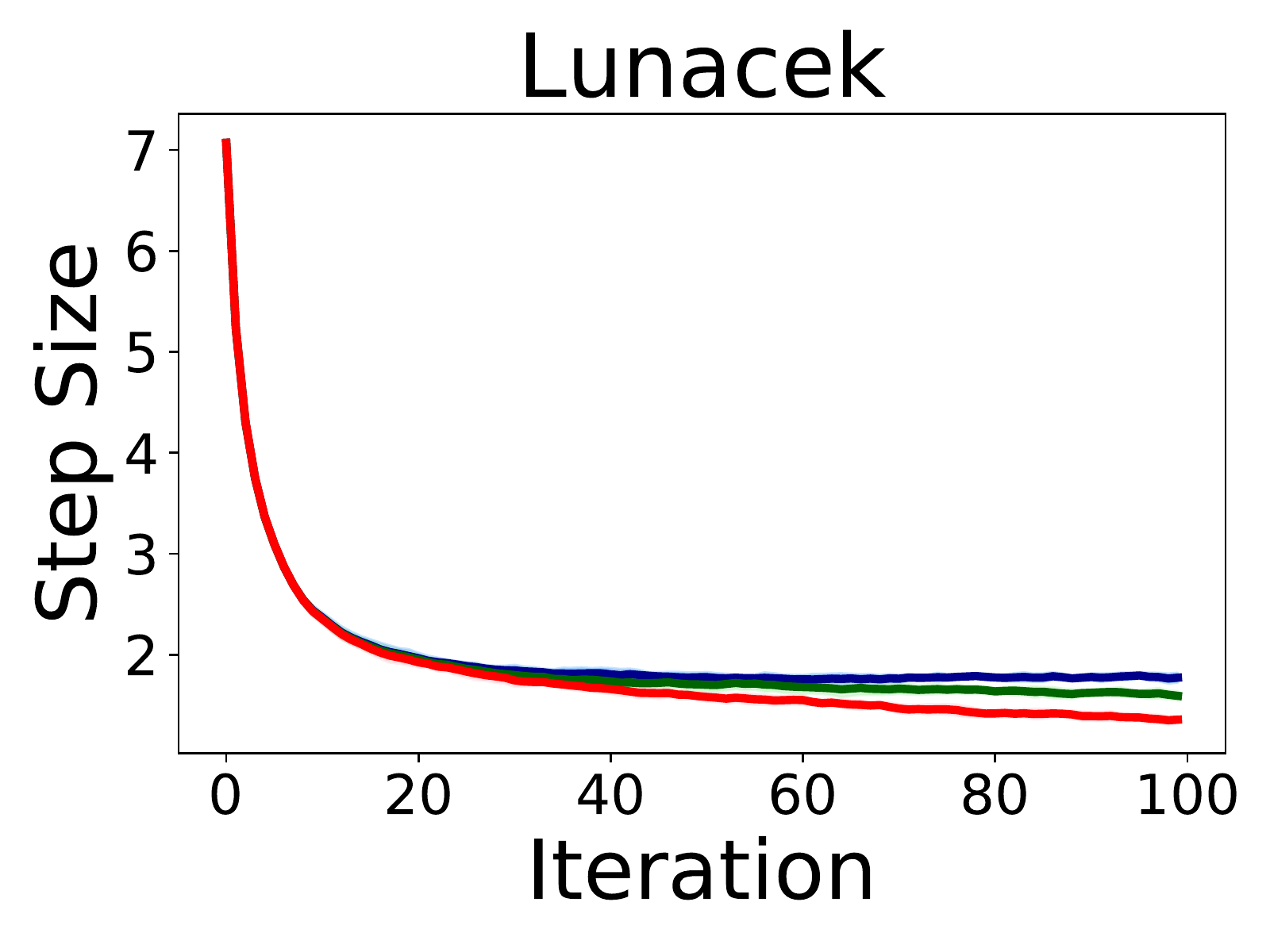}
\includegraphics[width=0.5\textwidth]{figures/legend.pdf}
\vskip -5pt
\caption{\small Average step size (solid) with standard deviation (shaded region) of the belief distribution's mean across 10 seeds for ES, NES, CMA, and CoNES on \texttt{Sphere, Rosenbrock, Rastrigin,} and \texttt{Lunacek}. \label{fig:benchmark-stepsize}}
\end{figure*}

\vspace{-0.2cm}
\subsection{Reinforcement Learning Tasks}
\label{subsec:RL-results}

Next, we benchmark our approach on the following environments from the OpenAI Gym suite of RL problems: \texttt{HalfCheetah-v2, Walker2D-v2, Hopper-v2}, and \texttt{Swimmer-v2}. We employ a fully-connected neural network policy with \texttt{tanh} activations possessing one hidden layer with 16 neurons for \texttt{Swimmer-v2} and 50 neurons for all other environments. The input to the policies are the agent's state -- which are normalized using a method similar to the one adopted by \cite{Mania18} -- and the output is a vector in the agent's action space. The training for these tasks was performed on a \texttt{c5.24xlarge} instance on Amazon Web Services (AWS).
Fig.~\ref{fig:RL-reward} presents the average and standard deviation of the rewards for each RL task across 10 seeds against the number of time-steps interacted with the environment. Fig.~\ref{fig:RL-reward} as well as Table~\ref{tab:timesteps} illustrate that CoNES performs well on all these tasks. For each environment we share the same hyperparameters (excluding $\epsilon$) between ES, NES, and CoNES; for CMA we use the default hyperparameters as chosen by PyCMA. It is worth pointing out that for RL tasks, CoNES demonstrates high sensitivity to the choice of $\epsilon$. The results for CoNES reported in Fig.~\ref{fig:RL-reward} and Table~\ref{tab:timesteps} are for the best choice of $\epsilon$ from $[\sqrt{0.1},\sqrt{1},\sqrt{10},\sqrt{100},\sqrt{1000}]$. Exact hyperparameters for the problems are provided in Appendix~\ref{app:hyperparams}. Each seed of \texttt{HalfCheetah-v2, Walker2D-v2} and \texttt{Hopper-v2}, takes $\sim$4-5 hours with ES, NES, CoNES and $\sim$10 hours with CMA. Each seed of \texttt{Swimmer-v2} takes $\sim$2 hours with ES, NES, CoNES and $\sim$4 hours with CMA.

\begin{figure*}
\vskip -10pt
\centering
\includegraphics[width=0.24\textwidth]{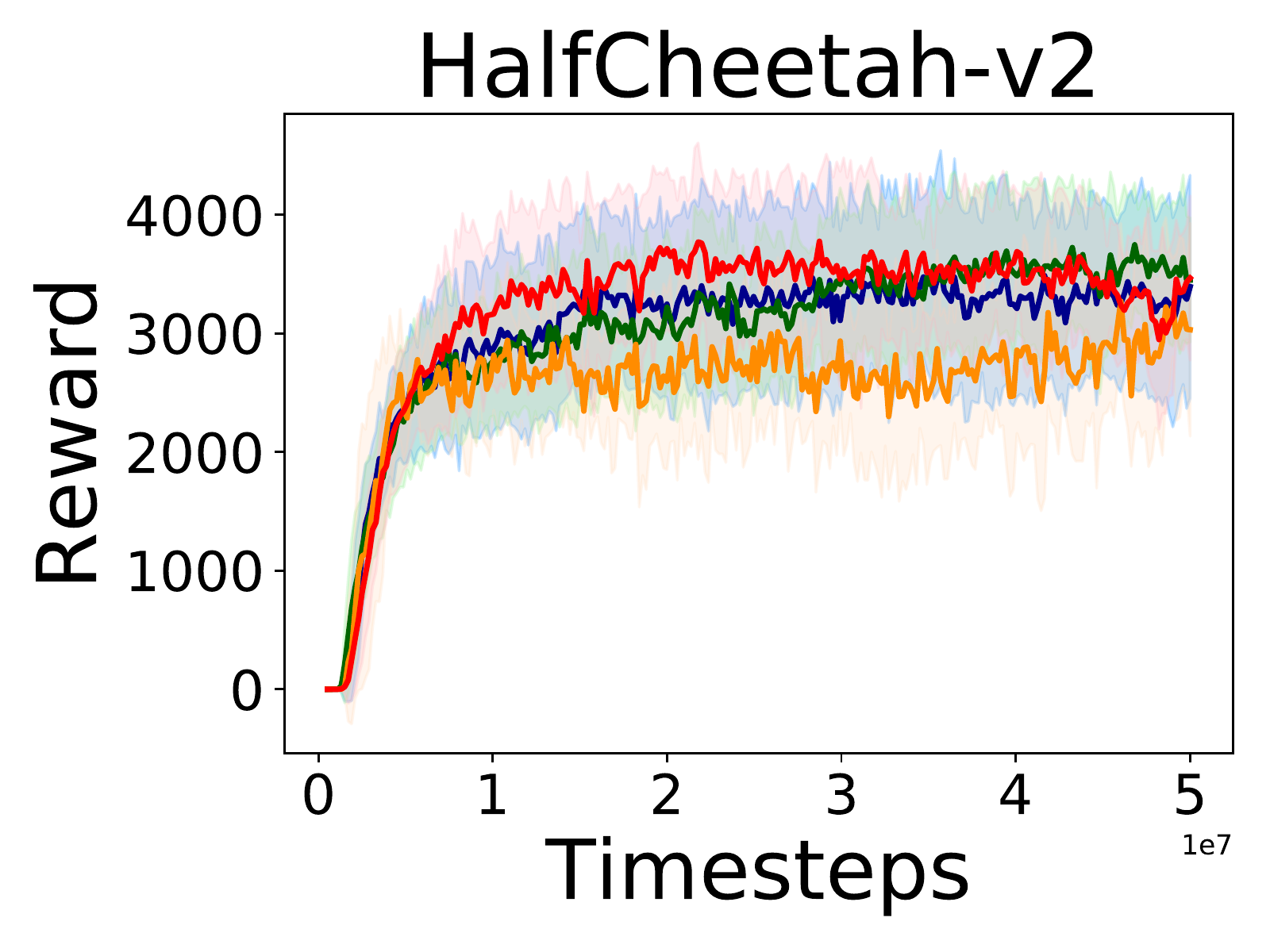}
\includegraphics[width=0.24\textwidth]{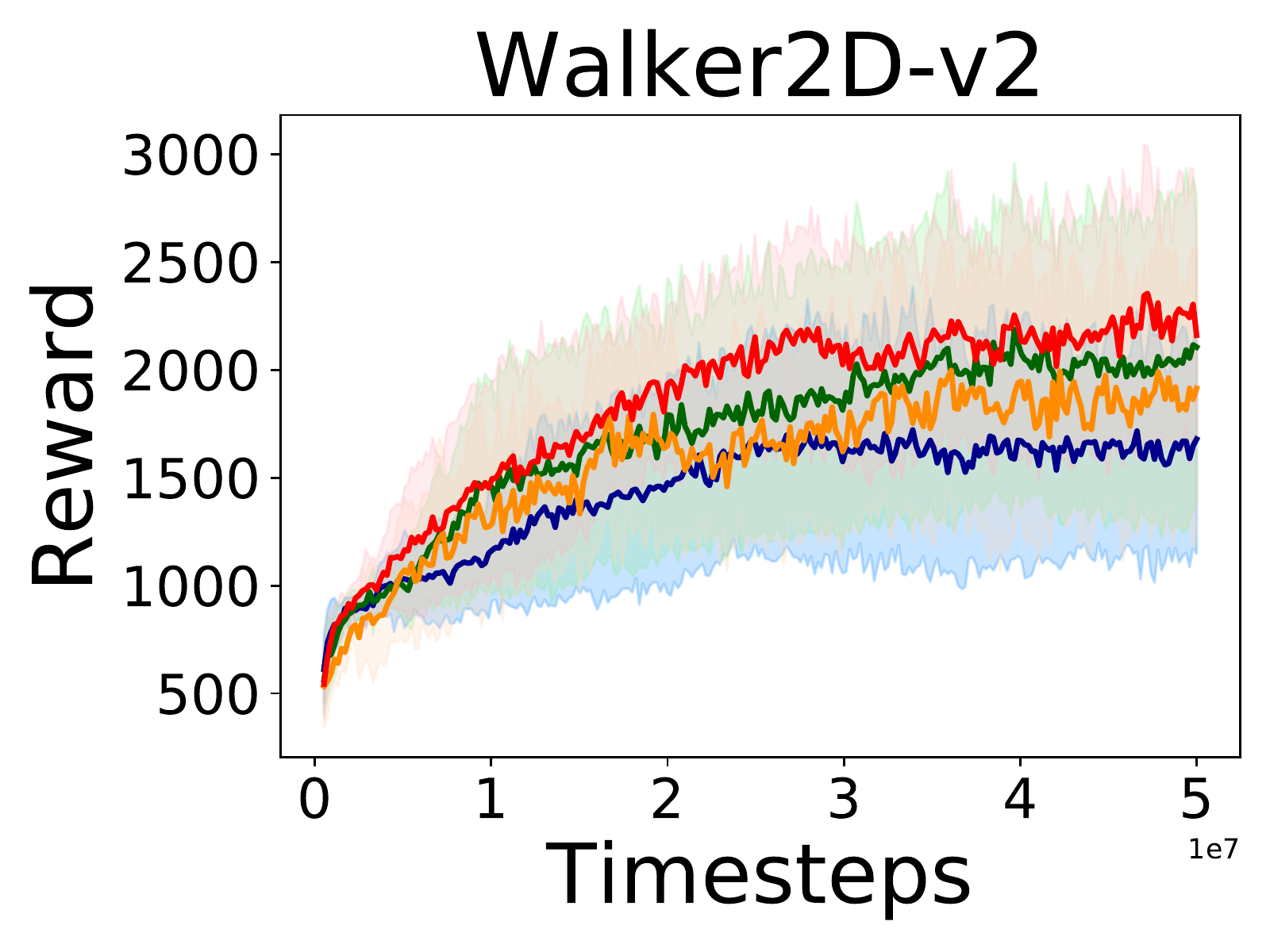}
\includegraphics[width=0.24\textwidth]{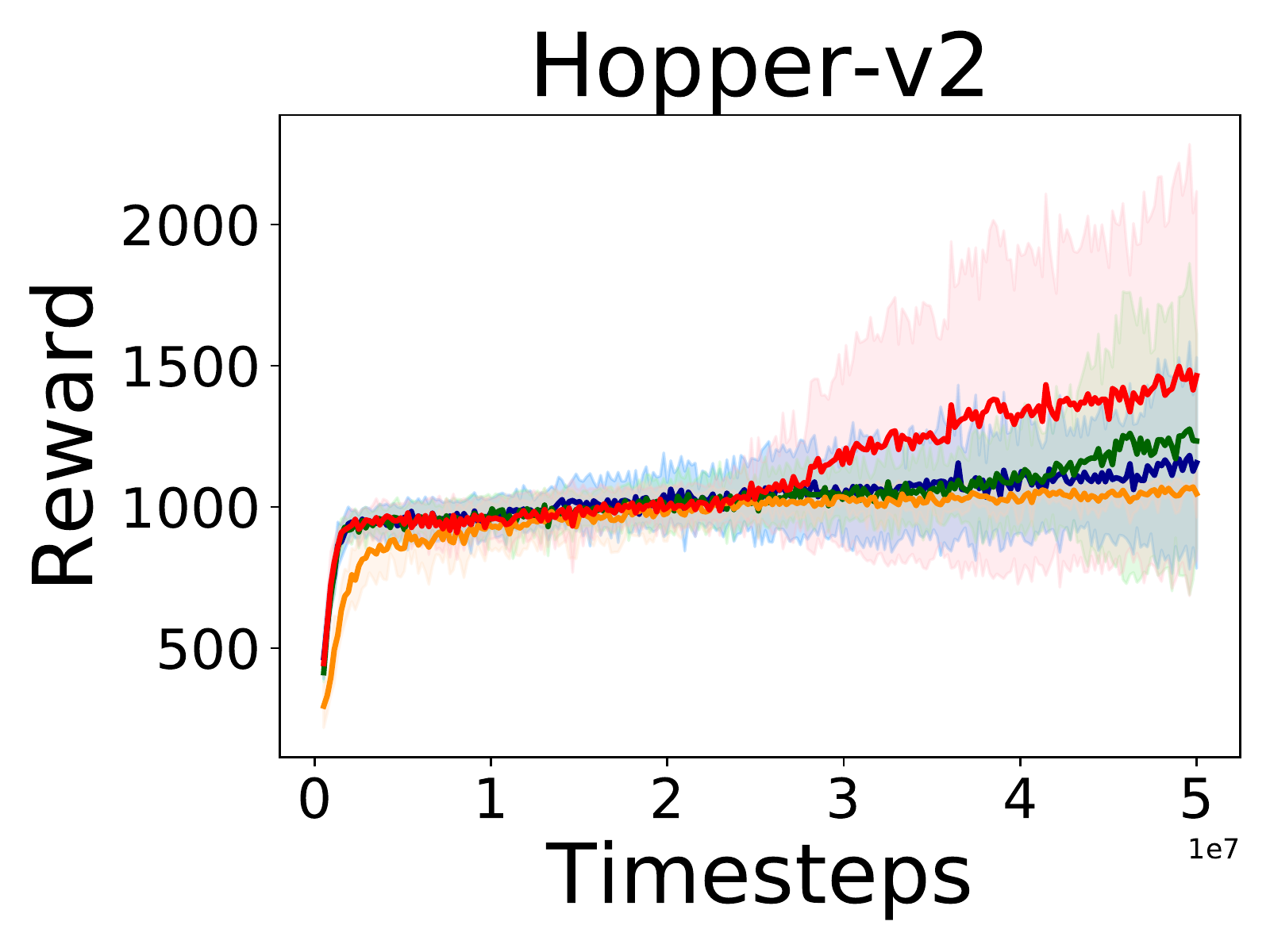}
\includegraphics[width=0.24\textwidth]{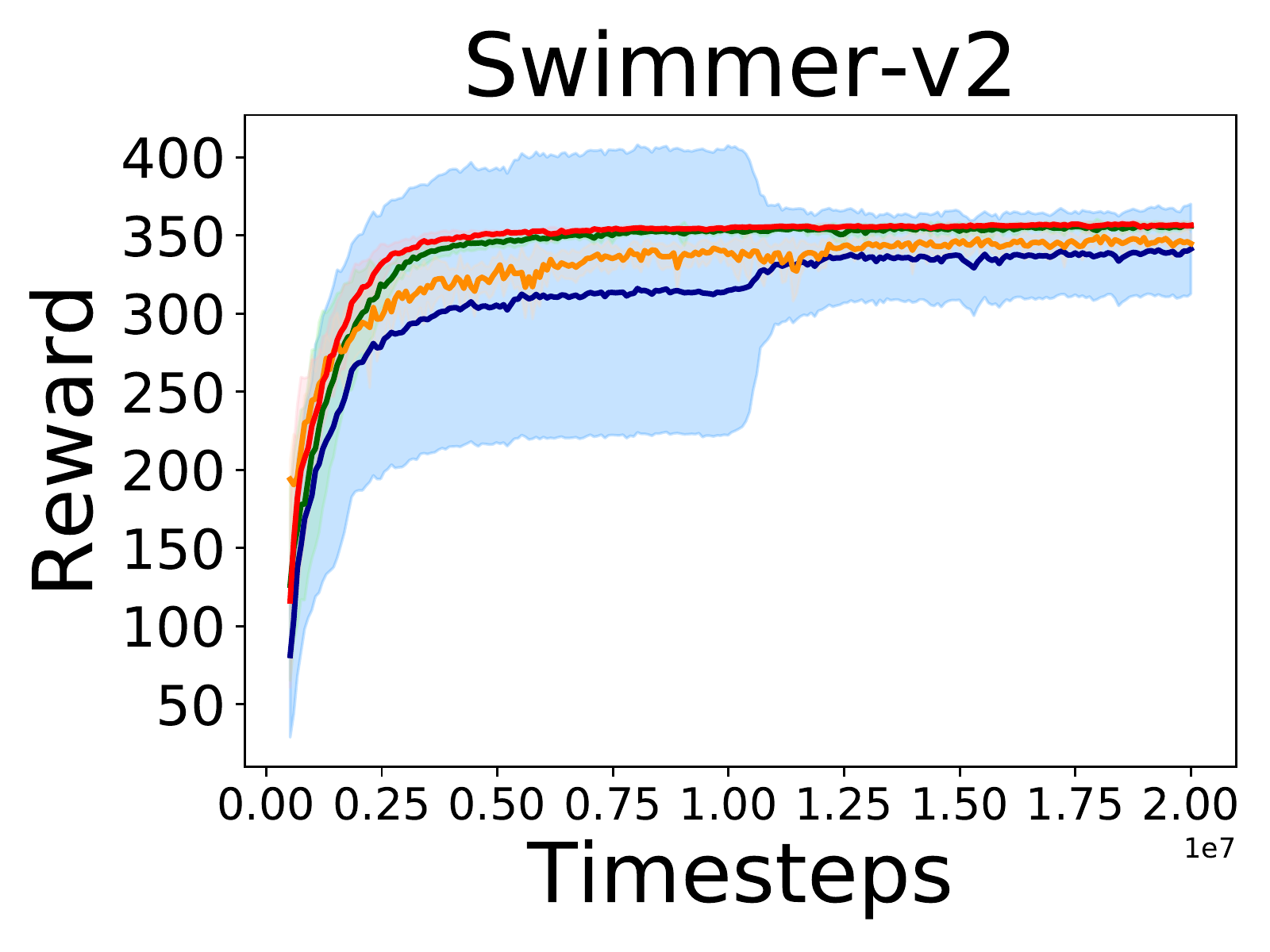}
\includegraphics[width=0.5\textwidth]{figures/legend.pdf}
\vskip -5pt
\caption{\small Average reward (solid curve) with standard deviation (shaded region) across 10 seeds for ES, NES, CMA, and CoNES on \texttt{HalfCheetah-v2, Walker2D-v2, Hopper-v2}, and \texttt{Swimmer-v2}. \label{fig:RL-reward}}
\vskip -5pt
\end{figure*}

\begin{table}
  \centering
  \begin{adjustbox}{width=1\columnwidth,center}
  \begin{tabular}{cccccc}
    \hlinewd{0.75pt}
     &  & \multicolumn{4}{c}{\# Timesteps to attain target average reward}\\
    Environments  & Target Avg. Reward & ES & NES & CMA & CoNES \\
    \hlinewd{0.75pt}
    \multicolumn{1}{l}{\rule{0pt}{3ex}\texttt{HalfCheetah-v2}} & 3500 & $3.23\times 10^7$ & $3.01\times10^7$ & -- & $\mathbf{1.40\times 10^7}$\\
    \multicolumn{1}{l}{\texttt{Walker2D-v2}} & 2000 & -- & $3.07\times 10^7$ & -- & $\mathbf{2.10\times10^7}$ \\
    \multicolumn{1}{l}{\texttt{Hopper-v2}} & 1400 & -- & -- & -- & $\mathbf{4.15\times 10^7}$ \\
    \multicolumn{1}{l}{\texttt{Swimmer-v2}} & 340 & $1.99\times 10^7$ & $3.60\times 10^6$ & $8.33\times 10^6$ & $\mathbf{3.01\times 10^6}$ \\
    \hlinewd{0.75pt}
  \end{tabular}
  \end{adjustbox}
  \vspace{2mm}
  \caption{Timesteps to attain a target average reward (over 10 seeds) for RL tasks. For each environment the timestep for the best performing blackbox method is displayed in bold. Hyphen ( -- ) is used for the method that failed to achieve the target average reward in $2\times 10^7$ timesteps for \texttt{Swimmer-v2} and $5\times 10^7$ timesteps for all other environments. \label{tab:timesteps}}
  \vspace{-5mm}
\end{table}

\section{Conclusions and Future Work}

We presented convex natural evolutionary strategies (CoNES) for optimizing high-dimensional blackbox functions. CoNES combines the notion of the natural gradient from information geometry with powerful techniques from convex optimization (e.g., second-order cone programming and geometric programming). In particular, CoNES refines a gradient estimate by solving a convex program that searches for the direction of steepest ascent in a KL-divergence ball around the current belief distribution. 
We formally established that CoNES is invariant under transformations of the belief parameterization. Our numerical results on benchmark functions and RL examples demonstrate the ability of CoNES to converge faster than conventional blackbox methods such as ES, NES, and CMA.  

\textbf{Future Work.} This paper raises numerous exciting future directions to explore. The performance of CoNES is dependent on the choice of the radius $\epsilon^2$ of the KL-divergence ball. Furthermore, a suitable choice of $\epsilon$ in one region of the loss landscape may not be suitable for another. Hence, an adaptive scheme for choosing the radius of the KL-divergence ball could substantially enhance the performance of CoNES. Another potentially fruitful future direction arises from the observation that Proposition~\ref{prop:nat-grad} --- which serves as the cornerstone of CoNES --- holds for any\footnote{This an outcome of the fact that the Hessian of all $f$-divergences is the Fisher information \cite{Makur15}.} $f$-divergence \cite{Csiszar04}. Hence, we can generalize CoNES to arbitrary $f$-divergences; this may afford greater flexibility in tuning it for the specific loss landscape and further improving performance. We can increase the flexibility afforded by CoNES even more by expanding beyond the family of Gaussian belief distributions. Finally, we are also exploring the empirical benefits of adaptively restricting the covariance matrix model \cite{Akimoto16, Choromanski19b} in order to further enhance sample complexity.

\section*{Acknowledgements}

The authors were supported by the Office of Naval Research [Award Number: N00014-18-1-2873], the Google Faculty Research Award, and the Amazon Research Award. 

\appendix
\section*{Appendix}

\section{Hyperparameters}
\label{app:hyperparams}

The parameters for the Adam optimizer were chosen according to \cite[Algorithm~1]{Kingma14} for all results in Section~\ref{sec:results}.  

\textbf{Benchmark Functions.} For all the results in Section~\ref{subsec:benchmark-results} the initial belief distribution is chosen to be the normal distribution $\mathcal{N}(0,I)$. The hyperparameters for ES, NES and CoNES were chosen as follows: the number of function evaluations performed per iteration is 100 and the learning rate for the mean and log of the variance is 0.1. Additionally, $\epsilon$ is set to 100 for CoNES.

\textbf{RL Tasks.} The hyperparameters for ES, NES, and CoNES for the results in Section~\ref{subsec:RL-results} are detailed in Table~\ref{tab:RL-hyperparams} below; some of these hyperparameters were borrowed from \cite{Pagliuca19}.

\begin{table}[h]
  \centering
  \begin{adjustbox}{width=1\columnwidth,center}
  \begin{tabular}{cccccccc}
    \hlinewd{0.75pt}
    & \multicolumn{2}{c}{Initial Distribution} & \multicolumn{2}{c}{Learning Rate} & \# policies evaluated & \# envs interacted \\
    Environments & mean ($\mu$) & std ($\sigma$)  & $\mu$ & $\log(\sigma^2)$ &  per itr (N) & per policy (m) & $\epsilon$ \\
    \hlinewd{0.75pt}
    \multicolumn{1}{l}{\rule{0pt}{3ex}\texttt{HalfCheetah-v2}} & 0 & 0.02 & 0.01 & 0.01 & 40 & 1 & $\sqrt{1000}$\\
    \multicolumn{1}{l}{\texttt{Walker2D-v2}} & 0 & 0.02 & 0.01 & 0.01 & 40 & 1 & $\sqrt{1000}$\\
    \multicolumn{1}{l}{\texttt{Hopper-v2}} & 0 & 0.02 & 0.01 & 0.01 & 40 & 1 & 1\\
    \multicolumn{1}{l}{\texttt{Swimmer-v2}} & 0 & 1 & 0.5 & 0.1 & 40 & 1 & 10\\
    \hlinewd{0.75pt}
  \end{tabular}
  \end{adjustbox}
  \vspace{2mm}
  \caption{Hyperparameters for RL tasks. \label{tab:RL-hyperparams}}
  \vspace{-5mm}
\end{table}

\section{Benchmark functions}
\label{app:benchmark}
Let $x\in\mathbb{R}^n$ be expressed in its coordinates as $x=(x_1,\cdots,x_n)$.
\begin{itemize}
    \item \texttt{Sphere}: $x \mapsto x^{\rm T}x$
    \item \texttt{Rosenbrock}: $x \mapsto \sum_{i=1}^{n-1}(100(x_i^2 - x_{i+1})^2 + (1-x_i)^2)$
    \item \texttt{Rastrigin}: $x \mapsto 10n + \sum_{i=1}^n(x_i^2 - 10 \cos(2\pi x_i))$
    \item \texttt{Lunacek}: First define the constants
    \begin{align}\nonumber
        \mu_1 = 2.5, ~~~~ s = 1- \frac{1}{2\sqrt{n+20}-8.2}, ~~~~ d=1, ~~~~ \mu_2 = -\sqrt{\frac{\mu_1^2 - d}{s}}.
    \end{align}
    Using these constants the function can be expressed as $x\mapsto \min\{ \sum_{i=1}^n(x_i - \mu_1)^2, dn + s\sum_{i=1}^n(x_i-\mu_2)^2 \} + 10\sum_{i=1}^n (1-\cos(2\pi(x_i-\mu_1)))$.
\end{itemize}

\bibliographystyle{abbrv}
\bibliography{irom}

\end{document}